\documentclass[letterpaper, 10 pt, conference]{ieeeconf}  

\IEEEoverridecommandlockouts                              

\usepackage{graphics} 
\usepackage{epsfig} 
\usepackage{mathptmx} 
\usepackage{times} 
\usepackage{amsmath} 
\usepackage{amssymb}  
\usepackage{xcolor}
\usepackage{booktabs}
\usepackage{bm}

\DeclareMathOperator*{\argmax}{arg\,max}

\newtheorem{theorem}{Theorem}
\newtheorem{corollary}{Corollary}

\title{\LARGE \bf
Trust-Aware Embodied Bayesian Persuasion for Mixed-Autonomy
}

\author{Shaoting Peng$^{1*}$, Katherine Driggs-Campbell$^1$, Roy Dong$^2$
\thanks{$^*$ Corresponding author: {\tt\small peng33@illinois.edu}}
\thanks{$^{1}$Electrical \& Computer Engineering, University of Illinois Urbana-Champaign, IL, 61801, USA}
\thanks{$^{2}$Industrial \& Enterprise Systems Engineering, University of Illinois Urbana-Champaign, IL, 61801, USA}
}

\begin{document}

\maketitle
\thispagestyle{empty}
\pagestyle{empty}

\begin{abstract}
Safe and efficient interaction between autonomous vehicles (AVs) and human-driven vehicles (HVs) is a critical challenge for future transportation systems. While game-theoretic models capture how AVs influence HVs, they often suffer from a long-term decay of influence and can be perceived as manipulative, eroding the human's trust. This can paradoxically lead to riskier human driving behavior over repeated interactions.
In this paper, we address this challenge by proposing the Trust-Aware Embodied Bayesian Persuasion (TA-EBP) framework. Our work makes three key contributions: First, we apply Bayesian persuasion to model communication at traffic intersections, offering a transparent alternative to traditional game-theoretic models. Second, we introduce a trust parameter to the persuasion framework, deriving a theorem for the minimum trust level required for influence. Finally, we ground the abstract signals of Bayesian persuasion theory into a continuous, physically meaningful action space, deriving a second theorem for the optimal signal magnitude, realized as an AV's forward nudge.
Additionally, we validate our framework in a mixed-autonomy traffic simulation, demonstrating that TA-EBP successfully persuades HVs to drive more cautiously, eliminating collisions and improving traffic flow compared to baselines that either ignore trust or lack communication. Our work provides a transparent and non-strategic framework for influence in human-robot interaction, enhancing both safety and efficiency.
\end{abstract}

\section{INTRODUCTION}

With the advancement of self-driving technology, our traffic systems are increasingly populated by autonomous vehicles (AVs). This trend is set to accelerate; according to the World Economic Forum, it is projected that by 2035 in the U.S., 32\% of new car sales will be for vehicles equipped with L2+ autonomy levels \cite{wef2025autonomous}. The proliferation of AVs in these mixed-autonomy environments will fundamentally reshape many aspects of traffic, most notably the nature of communication and interaction between AVs and human-driven vehicles (HVs).

Traditionally, research in AV decision-making has often adopted a reactive stance, modeling HVs as dynamic obstacles to be predicted and avoided without any explicit communicative actions \cite{Gray2013RobustPC, Vitus2013APA}. Such approaches, while safe, often lead to overly conservative AV behavior and reduced traffic efficiency. More recent work has explored the interactive nature of traffic by framing the AV-HV relationship from a game-theoretic perspective, often as a Stackelberg game \cite{Sadigh2016PlanningAC, Fisac2019HierarchicalGT}. These models more accurately capture how an AV's actions can influence an HV's decisions, leading to more assertive and efficient trajectories. However, a critical limitation has emerged: while effective in initial encounters, these game-theoretic strategies tend to lose their influence over repeated interactions. This decay of influence can inadvertently encourage HVs to drive more recklessly over time, increasing long-term accident risk \cite{Cooper2019StackelbergPB, Sagheb2023TowardsRI}.

To maintain this influence, existing literature has proposed injecting randomness into the AV's strategy, for example by adding Gaussian noise to its actions or introducing different entropy terms into its reward function \cite{Sagheb2023TowardsRI}. While these methods can prolong the AV's influence, they do so by intentionally obscuring the AV's true intent. This reliance on being ``unpredictable" can be criticized as a form of strategic dishonesty. Alternative, non-game-theoretic approaches, such as constructing complex dynamic human models with tools like MOMDPs \cite{Sagheb2025UnifiedFR}, often suffer from prohibitive computational complexity and rely on strong modeling assumptions, limiting their real-world applicability.

To address this critical gap, we propose a framework that is both influential and transparent, capable of guiding HVs toward safer behavior without resorting to manipulation. This objective motivates a shift in perspective toward the principle of \textit{influencing human beliefs}. A key distinction from many game-theoretic approaches, which often focus on optimizing a single strategic action in real-time, is our focus on designing a complete \textit{signaling scheme}: a probabilistic mapping from the AV's intent to its emitted signal to which the AV publicly and transparently commits beforehand.
To illustrate, consider a four-way intersection where an AV and an HV arrive simultaneously. A deterministic strategy, such as always nudging forward as a signal, would be quickly learned by the HV, rendering the signal uninformative and even encouraging reckless behavior. In contrast, our approach involves committing a public probabilistic signaling scheme. By carefully designing \textit{when} and \textit{how} to signal based on this public policy, the AV can shape the HV's posterior beliefs. This allows the AV to remain influential over repeated interactions by transparently managing how its actions provide information, thereby guiding the HV toward safer behavior without resorting to deception.

To formalize this, in this work we employ \textit{Bayesian persuasion} \cite{Kamenica2011BayesianP}, a framework from economics that provides a formal method for influencing a rational agent's behavior by committing to a transparent signaling scheme. However, applying this framework to our embodied, human-in-the-loop scenario presents two major challenges.

First, the framework's assumption of perfect rationality is fragile, as human trust towards such AV signals is a critical mediating factor \cite{Choi2015InvestigatingTI}. An untrusting driver, for example, will not update their beliefs in a perfectly rational Bayesian manner. While considerable research has focused on modeling trust itself, such as a partially observable state in a POMDP \cite{Sheng2022PlanningFA} or a hidden variable within a Dynamic Bayesian Network \cite{Xu2015OPTIMO}, the mechanism by which trust modulates belief formation remains underexplored \cite{Rodrigues2023ARO}. We tackle this limitation by enhancing the persuasion model with a \textit{trust level} parameter. This parameter, $\theta \in [0, 1]$, explicitly models the degree to which a human incorporates the AV's signal, acting as a weight between their prior beliefs and the Bayesian posterior. This formulation generalizes the original framework, making it robust to varying and dynamic levels of human trust. We further develop a theorem on minimum persuadable trust level, and thus provide answer to the \textit{when to signal} question.

Second, a critical challenge is the design and optimization of the embodied signals themselves. Existing literature typically models signals as either abstract symbols \cite{Dughmi2016AlgorithmicBP, Kamenica2019BayesianPA} or as draws from a predefined finite set, such as binary clinical test results \cite{Kolotilin2015ExperimentalDT} or high/low financial scores \cite{Goldstein2018StressTA}. Crucially, these works do not provide a principled method for determining the optimal physical values of the signals. 
In contrast, a framework for real-world embodied AI requires that signals be mapped to a continuous space of physically meaningful and interpretable actions. We therefore investigate how to ground the signals of our trust-aware persuasion framework in an AV's continuous behaviors, such as subtle nudges, ensuring they are both optimally informative and physically realizable. In this way we provide answer to the \textit{how to signal} question.

To overcome these limitations, we propose the Trust-Aware Embodied Bayesian Persuasion (TA-EBP) framework, which influences an HV's behavior by shaping their posterior beliefs in a transparent and provably non-strategic manner. Specifically, our work makes three key contributions:
\begin{enumerate}
    \item We apply Bayesian persuasion to model communication at traffic intersections, offering a transparent alternative to traditional game-theoretic models.
    \item We augment the classical persuasion model with a trust parameter to account for variance in human rationality, and we derive a theorem for the minimum trust required for influence. 
    \item We ground the framework in the physical world by modeling signals as continuous, embodied actions and derive a second theorem for the optimal AV nudge magnitude. 
\end{enumerate}

We validated our TA-EBP framework in a mixed-autonomy simulation against non-communicative and trust-agnostic baselines. While the baselines led to traffic congestion or predictable collisions, our TA-EBP successfully calibrated its signal to the human's trust level, which resulted in zero collisions and a significant increase in safe driving behavior, enhancing both traffic safety and efficiency.

\section{Preliminary: Bayesian Persuasion}

Bayesian persuasion is a strategic framework that studies how an informed agent (call him Sender) can influence the behavior of another rational agent (call her Receiver) by strategically revealing information \cite{Kamenica2011BayesianP}. Initially developed for economics, the framework has proven widely applicable to diverse fields including politics \cite{Andrew2023BayesianEP}, marketing \cite{Tamer2024PersuadingSF}, education \cite{Boleslavsky2015GradingSE}, and security \cite{Rabinovich2015InformationDA}. The core idea is that the Sender commits to a \textit{signaling scheme} that maps the true state of the world to a public signal. The Receiver observes this signal, updates her beliefs via Bayes' rule, and takes an action to maximize her own utility. The Sender's challenge is to design the scheme that persuades the Receiver to choose actions favorable to the Sender.


\subsection{Formal Model}
The model is defined by the following components:
\begin{itemize}
    \item A state of the world $\omega \in \Omega$ and an action for the Receiver $a \in \mathcal A$. For simplicity, we assume the state and action spaces $\Omega$ and $A$ are finite and discrete sets.
    \item A common prior belief $p_0 \in \text{int}(\Delta (\Omega))$, where $\Delta(X)$ is the probability simplex over a set $X$ and $\text{int}(\cdot)$ denotes its interior.
    \item Utility functions for the Sender, $v(\omega, a)$, and the Receiver, $u(\omega, a)$, which are common knowledge among both Sender and Receiver.
    \item A set of signals $\mathcal S$ and a signaling scheme $\pi: \Omega \to \Delta(\mathcal S)$. The signaling scheme $\pi(s|\omega)$ specifies the probability of the Sender sending signal $s$ given the true state is $\omega$.\footnote{From a probabilistic viewpoint, the signal $s$ is an observation, and the signaling scheme $\pi(s|\omega)$ is the likelihood of that observation given the state.}
\end{itemize}

The interaction unfolds based on a strict timeline. First, the Sender designs and commits to the signaling scheme $\pi$ \textit{before} observing the true state of world. Then, nature draws a state $\omega \sim p_0(\omega)$, which the Sender observes. The Sender then sends a signal $s \sim \pi(\cdot|\omega)$. Upon observing the signal $s$, the rational Receiver updates her belief about the state of the world from the prior $p_0(\omega)$ to a posterior $p(\omega|s)$ using Bayes' rule:
\begin{equation}
    p(\omega|s) = \frac{\pi(s|\omega)p_0(\omega)}{\sum_{\omega' \in \Omega} \pi(s|\omega')p_0(\omega')}
    \label{eq:bayes_update}
\end{equation}

Given this posterior belief, she chooses an action $a^*$ that maximizes her own expected utility:
\begin{equation}
    a^*(s) = \argmax_{a \in \mathcal A} \sum_{\omega \in \Omega} p(\omega|s) u(\omega, a)
    \label{eq:receiver_action}
\end{equation}

The Sender's problem is to design the optimal signaling scheme $\pi$ that maximizes his own expected utility, anticipating the Receiver's best response from Equation~\eqref{eq:receiver_action}:
\begin{equation}
    \pi^* = \argmax_\pi \sum_{\omega \in \Omega} p_0(\omega) \sum_{s \in \mathcal S} \pi(s|\omega) v(\omega, a^*(s))
\end{equation}

In summary, the persuasion process follows a clear sequence \cite{Dughmi2016AlgorithmicBP}:
\begin{enumerate}
    \item \textbf{Setup:} The Sender and Receiver know the prior $p_0(\omega)$ and both utility functions, $u$ and $v$.
    \item \textbf{Commitment:} The Sender designs and commits to a signaling scheme $\pi(s|\omega)$, which the Receiver is fully-aware of.
    \item \textbf{Realization:} The true state $\omega$ is realized, and the Sender draws and transmits a signal realization $s$ according to $\pi(s|\omega)$.
    \item \textbf{Action:} The Receiver observes $s$, updates her beliefs according to Equation~\eqref{eq:bayes_update}, and chooses an action $a^*$ that maximizes her utility according to Equation~\eqref{eq:receiver_action}.
\end{enumerate}

\subsection{Bayesian Persuasion Example}
\label{sec:BP_example}

To concretize the framework, we consider a canonical example from the literature. Suppose a prosecutor (Sender) seeks to persuade a judge (Receiver) to convict a defendant. The state of the world $\omega$ is that the defendant is either \textit{guilty} or \textit{innocent}, with a shared prior belief of $p_0(guilty)=0.3$ and $p_0(innocent)=0.7$. The judge's action space $\mathcal A$ consists of \{\textit{convict}, \textit{acquit}\}; her utility is 1 for a correct judgment (i.e., convict if guilty and acquit if innocent) and 0 otherwise. However, the prosecutor's utility is 1 for a conviction and 0 otherwise, irrespective of the true state. The prosecutor may commit to a signaling scheme $\pi$ -- representing an investigation process -- which maps the true state to a signal $s \in \{g, i\}$ that is honestly reported to the judge.

Under this setting, if no information is communicated, the judge's decision is dictated by the prior, leading her to \textit{acquit}. If the prosecutor commits to a fully informative investigation that perfectly reveals the state, the probability of conviction would be 0.3. It can be demonstrated, however, that the optimal persuasion policy is the following partially informative signaling scheme:
\begin{align*}
    \pi(i|innocent) = 4/7, \quad & \pi(g|{innocent}) = 3/7 \\
    \pi(i|guilty) = 0, \quad & \pi(g|{guilty}) = 1
\end{align*}

The posterior beliefs induced by this scheme are as follows:
\begin{align*}
    p({innocent}|i) = 1, \quad & p({guilty}|i) = 0 \\
    p({innocent}|g) = 0.5, \quad & p({guilty}|g) = 0.5
\end{align*}

Under this scheme, the probability of conviction rises to 0.6. This result is notable because the judge, despite being fully cognizant that the defendant is a priori likely innocent and that the investigation was designed explicitly to maximize the conviction rate, is nonetheless rationally compelled to convict upon observing the signal $g$ (assuming break ties in favor of Sender).

In contrast to conventional game-theoretic models, Bayesian persuasion is particularly well-suited for our work due to its inherent transparency and non-strategic nature. The Receiver is fully informed of the Sender's utility function and the signaling scheme, yet can still be influenced toward the Sender's preferred outcome. Moreover, the Sender's commitment to the signaling scheme is made \textit{ex-ante}, before the realization of the true state. This structure precludes opportunistic deception and fosters the kind of predictability and trust essential for our application. 
While prior studies have invoked Bayesian persuasion in traffic system communication applications \cite{Peng2019BayesianPD, Paul2024MonotonicAP}, they often overlook its core mechanism of the persuasive nature of the signaling scheme. Our work aims to provide a more principled framework.\footnote{Note that this is an incomplete but essential introduction to Bayesian persuasion in order to understand this work, and many concepts such as being \textit{Bayes plausible} are ignored. Please refer to \cite{Kamenica2011BayesianP} for more details.}

\section{Model}

In this section, we formalize the model used in this work. The interaction scenario, inspired by prior work in multi-agent planning \cite{Sadigh2016PlanningAC, Sagheb2023TowardsRI}, involves an autonomous vehicle (AV) and a human-driven vehicle (HV) arriving at a two-way stop intersection (single lane) at approximately the same time. The AV, as the informed Sender, has a private intent to either proceed through the intersection or wait. The HV, as the rational Receiver, must decide on an action based on her beliefs about the AV's intent.

The components of our persuasion game are defined as follows:
\begin{itemize}
    \item States ($\Omega$): The state of the world is the AV's private intent, $\omega \in \{{Go}, {Stop}\}$. The common prior is $p_0({Go})=\lambda$ and $p_0(Stop)=1-\lambda$.
    
    \item Actions ($\mathcal A$): The HV can choose an action $a \in \{\text{Drive Recklessly}~(DR), \text{Drive if Clear}~(DC)\}$.
    
    \item Utilities ($u, v$): The HV's utility $u(\omega, a)$ reflects a trade-off between efficiency and safety. Driving recklessly is efficient (utility 1) if the AV intends to \textit{Stop}, but catastrophic (utility $-r$, where $r > 1$) if the AV intends to \textit{Go}. Driving if clear (\textit{DC}) is always safe but can be inefficient (utility -1) if the HV forgoes an opportunity to proceed. The AV's utility $v(\omega, a)$ is purely safety-oriented, yielding 1 if the HV \textit{DC} and -1 if the HV \textit{DR}, irrespective of its own intent. These utilities are summarized in Table~\ref{tab: utility}.
    
    \item Signals ($\mathcal S$): The AV communicates its intent via an embodied signal. It can either nudge forward a small distance $s$, or remain stationary, $s=0$. The design of the optimal nudge distance $s^*$ is discussed in Section~\ref{sec:s_star}.
\end{itemize}

\begin{table}[htbp]
  \centering
  \caption{\textnormal{\small{Utilities for HV (in \textcolor{red}{red}) and AV (in \textcolor{blue}{blue}).}}}
  \label{tab: utility}
  \begin{tabular}{|l|c|c|}
    \hline
    \textbf{} & (AV) {$\omega=Go$} & (AV) {$\omega=Stop$} \\
    \hline
    (HV) \textit{a=DR} & \textcolor{red}{-r} / \textcolor{blue}{-1} & \textcolor{red}{1} / \textcolor{blue}{-1} \\
    \hline
    (HV) \textit{a=DC} & \textcolor{red}{0} / \textcolor{blue}{1} & \textcolor{red}{-1} / \textcolor{blue}{1} \\
    \hline
  \end{tabular}
\end{table}

Compared to traditional Bayesian persuasion work, a key difference of our model is the explicit inclusion of a \textbf{trust parameter}, $\theta \in [0, 1]$, which relaxes the assumption that the HV is a perfectly rational Bayesian updater. This parameter captures the heterogeneous nature of human trust in autonomous systems, which may be influenced by factors like prior experience or technical literacy. 
In this work, we assume the HV's trust level is a known and static parameter for a given interaction. This is a common simplification necessary for developing the foundational framework. We envision that a real-world system would initialize this parameter using a population average from prior studies and then, as outlined in our future work section, learn a driver-specific trust value over repeated interactions.

We model the HV's trust-aware posterior belief, $p_\theta(\omega|s)$, as a convex combination of her initial prior belief $p_0(\omega)$ and the fully-trusting Bayesian posterior $p_{FT}(\omega|s)$:
\begin{align}
    p_\theta(\omega|s) &= (1-\theta)p_0(\omega) + \theta p_{FT}(\omega|s) \\
    &= (1-\theta)p_0(\omega) + \theta \frac{\pi(s|\omega)p_0(\omega)}{\sum_{\omega' \in \Omega} \pi(s|\omega')p_0(\omega')}
\end{align}

The trust parameter $\theta$ interpolates between two behavioral extremes. When $\theta=0$, the HV completely distrusts the AV's signals, and her posterior belief remains identical to her prior: $p_{\theta=0}(\omega|s) = p_0(\omega)$. Conversely, when $\theta=1$, the HV fully trusts the signal as new evidence, behaving as a classic rational Bayesian agent: $p_{\theta=1}(\omega|s) = p_{FT}(\omega|s)$.

\section{Theoretical Analysis}

This section presents the theoretical analysis based on the model. Our TA-EBP framework addresses two central questions that arise: 1) Given the introduction of the trust parameter $\theta$, is there a subset of HVs that can never be persuaded, and can this be generalized? 2) How can the embodied signals be designed to be physically meaningful? We answer these questions by deriving two foundational theorems for the TA-EBP framework.

\subsection{Theorem: Minimum Persuadable Trust Level}

We first analyze the implications of augmenting the classical Bayesian persuasion framework with the trust parameter $\theta$. We demonstrate that there exists a unique trust threshold, $\theta^*$, below which it is impossible for the Sender to persuade the Receiver to deviate from her default action.

\begin{theorem}
For any persuasion setting with a trust level $\theta$, a Sender, and a Receiver with a default optimal action $a_0^*$, there exists a minimum persuadable trust threshold $\theta^*$ such that, if the Receiver's trust $\theta < \theta^*$, no signaling scheme can persuade her to select an action other than $a_0^*$. The general form of this threshold is given by:
\begin{equation}
\label{eq:theta_star}
    \theta^* = \min_{a' \in \mathcal A \setminus \{a_0^*\}} \left\{\frac{-\Delta U_0(a',a_0^*)}{\max_{s \in \mathcal S}\{\Delta U_{FT}(a',a_0^*;s)\}-\Delta U_0(a',a_0^*)}\right\}
\end{equation}
where $U(a|p) = \sum_\omega p(\omega)u(\omega, a)$ is the expected utility, and $\Delta U_0(a', a_0^*) = U(a'|p_0) - U(a_0^*|p_0)$ is the expected utility gain under the prior, $\Delta U_{FT}(a', a_0^*; s) = U(a'|p_{FT}(\cdot|s)) - U(a_0^*|p_{FT}(\cdot|s))$ is the expected utility gain under the fully-trusting posterior given signal $s$.
\end{theorem}

\begin{proof}
The Receiver's default action, $a_0^*$, is the one that maximizes her expected utility given the prior belief $p_0$:
\begin{equation}
    a_0^* = \argmax_{a \in \mathcal A} U(a|p_0) \quad \text{where} \quad U(a|p_0) = \sum_\omega p_0(\omega)u(\omega, a)
\end{equation}

For the Sender to successfully persuade the Receiver to take a different action, $a_k \neq a_0^*$, the expected utility of $a_k$ under the trust-aware posterior $p_\theta(\cdot|s)$ must be at least as great as that of $a_0^*$:
\begin{equation}
    U(a_k|p_\theta(\cdot|s)) \ge U(a_0^*|p_\theta(\cdot|s))
\end{equation}

Substituting the definition of the trust-aware posterior, $p_\theta(\omega|s) = (1-\theta)p_0(\omega)+\theta p_{FT}(\omega|s)$, we get:
\begin{equation}
\begin{split}
    \sum_\omega u(\omega, a_k)((1-\theta)p_0(\omega)+\theta p_{FT}(\omega|s)) \ge \\ \sum_\omega u(\omega, a_0^*)((1-\theta)p_0(\omega)+\theta p_{FT}(\omega|s))
\end{split}
\end{equation}

This inequality can be separated by $\theta$:
\begin{equation}
\begin{split}
    (1-\theta)U(a_k|p_0)+\theta U(a_k|p_{FT}) \ge \\
    (1-\theta)U(a_0^*|p_0)+\theta U(a_0^*|p_{FT})
\end{split}
\end{equation}

Rearranging the terms and using the $\Delta U$ notation yields:
\begin{equation}
    \theta \Delta U_{FT}(a_k, a_0^*;s) \ge -(1-\theta) \Delta U_{0}(a_k, a_0^*)
\end{equation}

Solving for $\theta$, we find the minimum trust required for a given signal $s$ to persuade the Receiver to choose $a_k$:
\begin{equation}
    \theta \ge \frac{-\Delta U_0(a_k,a_0^*)}{\Delta U_{FT}(a_k,a_0^*;s)-\Delta U_0(a_k,a_0^*)}
\end{equation}

To find the absolute minimum trust threshold for persuading the Receiver to take action $a_k$, the Sender must employ the most effective signal $s^*$, which is the one that maximizes the utility gain $\Delta U_{FT}$. If even this best-case signal fails, persuasion is impossible. Thus, the threshold for a specific alternative action $a_k$ is:
\begin{equation}
    \theta^*(a_k) = \frac{-\Delta U_0(a_k,a_0^*)}{\max_{s \in \mathcal S}\{\Delta U_{FT}(a_k,a_0^*;s)\}-\Delta U_0(a_k,a_0^*)}
\end{equation}

The overall minimum threshold $\theta^*$ is the lowest trust level required to make any of the alternative actions possible, which corresponds to the minimum of these thresholds over all alternative actions $a' \in \mathcal A \setminus \{a_0^*\}$, leading to the general theorem on the minimum persuadable trust level as in Equation~\eqref{eq:theta_star}.
\end{proof}

\subsubsection{Application to the Two-State, Two-Action Case}

To apply this theorem to our specific model, we present a corollary for the common case involving two states, actions, and signals.

\begin{corollary}
In a setting with states $\Omega=\{\omega_1, \omega_2\}$, actions $\mathcal A=\{a_1,a_2\}$, and default action $a_2$, the minimum persuadable trust threshold $\theta^*$ from Theorem 1 simplifies to:
\begin{equation}
\label{eq:theta_two_states}
    \theta^* = -\frac{p_0(\omega_1)\Delta u(\omega_1)+p_0(\omega_2)\Delta u(\omega_2)}{p_0(\omega_2)(\Delta u(\omega_1) - \Delta u(\omega_2))}
\end{equation}
where $s_1$ is the signal designed to be maximally informative about state $\omega_1$.
\end{corollary}
\begin{proof}
The term $\Delta U_0(a_1,a_2)$ is expanded based on the prior:
\begin{equation}
\label{eq:15}
    \Delta U_0(a_1,a_2) = p_0(\omega_1)\Delta u(\omega_1)+p_0(\omega_2)\Delta u(\omega_2)
\end{equation}
where $\Delta u(\omega_i)=u(\omega_i, a_1) - u(\omega_i, a_2)$ is the utility gain of $a_1$ over $a_2$ on state $\omega_i$. To resolve the $\max_{s}$ term in the general theorem, the Sender must choose the signal that is most persuasive towards action $a_1$. This is achieved by a signal $s_1$ that makes the posterior belief $p_{FT}(\omega_1|s_1)$ as high as possible, ideally a fully revealing signal for state $\omega_1$ such that $p_{FT}(\omega_1|s_1)=1$. In this case, $\Delta U_{FT}$ simplifies to:
\begin{equation}
\label{eq:16}
    \Delta U_{FT}(a_1,a_2;s_1) = u(\omega_1,a_1) - u(\omega_1, a_2) = \Delta u(\omega_1)
\end{equation}

Simplifying Equation~\eqref{eq:theta_star} with $a'=a_1, a_0^*=a_2$, and $s=s_1$ we have:
\begin{equation}
    \theta^* = \frac{-\Delta U_0(a_1,a_2)}{\Delta U_{FT}(a_1,a_2;s_1)-\Delta U_0(a_1,a_2)}
\end{equation}

Substituting $\Delta U_0(a_1,a_2)$ and $\Delta U_{FT}(a_1,a_2;s_1)$ from Equation~\eqref{eq:15} and \eqref{eq:16} into above provides the final threshold as Equation~\eqref{eq:theta_two_states}.
\end{proof}

For our AV-HV model, we map these variables as follows: $\omega_1={Go}, \omega_2={Stop}$; $a_1={DC}, a_2={DR}$. Plugging the utilities from Table~\ref{tab: utility} into the corollary, we derive the specific trust threshold for our scenario:
\begin{equation}
\label{eq:tt}
    \theta^* = \frac{2(1-\lambda) -r\lambda}{(1-\lambda)(r+2)}
\end{equation}

To analyze the behavior of this threshold, we compute its partial derivatives with respect to the prior $\lambda$ and the collision penalty $r$:
\begin{align}
    \frac{\partial \theta^*}{\partial \lambda} = -\frac{r}{(r+2)(1-\lambda)^2}, ~
    \frac{\partial \theta^*}{\partial r} = \frac{2}{(\lambda-1)(r+2)^2}
\end{align}

Given that $r>0$ and $\lambda \in (0,1)$, both partial derivatives are negative. This provides two key insights aligning with the intuition. First, a smaller $\lambda$ (a stronger prior belief that the AV will \textit{Stop}) requires a higher trust threshold $\theta^*$ to overcome. Second, a smaller penalty $r$ for collision also increases $\theta^*$, as the HV is less deterred by risk and thus harder to persuade.

\subsection{Theorem: Optimal Embodied Signal}
\label{sec:s_star}

Having established the condition under which an HV is persuadable, the next logical question is how to design the signal itself without a predefined set of signals. This requires grounding the abstract signals of theory in the physical world by deriving the optimal magnitude of an embodied action -- in our context, the specific nudge distance $s^*$.


We assume humans have an intuitive model of the AV's driving behavior, represented by a likelihood function $p(s|\omega)$. This reflects their belief about the probability of observing a nudge $s$ given the AV's intent $\omega$. Based on physical intuition, this model is monotonic: a larger nudge is more likely if the intent is \textit{Go}, so $p(s|Go)$ increases with $s$, while $p(s|Stop)$ decreases. A critical component of our framework is to ensure that the AV's announced signaling scheme $\pi$ is credible. To achieve this, we derive an optimal signal magnitude $s^*$ that aligns with the human's inherent beliefs. We enforce a consistency condition where the posterior belief calculated using our announced policy $\pi(s^*|\omega)$ is identical to the posterior calculated using the human's intuitive model $p(s^*|\omega)$. This alignment ensures that whether an HV follows our signaling framework or uses their own intuition to interpret the signal $s^*$, they arrive at the same posterior, thereby increasing the framework's credibility.

Let the likelihood ratio $L(s)=\pi(s|Go)/\pi(s|Stop)$, then the trust-aware posterior for the \textit{Go} state can thus be expressed as a function of $L(s)$:
\begin{equation}
\label{eq:pp}
    p_\theta({Go}|s)=(1-\theta)p_0({Go})+\theta \left( \frac{L(s)p_0({Go})}{L(s)p_0({Go})+p_0({Stop})}\right)
\end{equation}

Given that $p(s|Go)$ increases and $p(s|Stop)$ decreases, $L(s)$ is increasing in $s$. It follows that $p_\theta({Go}|s)$ is also strictly increasing in $s$, which means that the more the AV nudges forward, the more HV believes the AV intents to \textit{Go}, aligning with the intuition. To minimize costs associated with signaling (e.g., time, energy, or perceived aggressiveness), the AV's goal is to find the minimum effective signal value, $s^*$, that is \textit{just sufficient} to persuade the HV to switch from \textit{DR} to \textit{DC}.

This persuasion occurs at the indifference point, where the HV's expected utility for both actions is equal. We define the threshold posterior $p^*$ as the belief of \textit{Go} that satisfies:
\begin{equation}
    U({DR}|p^*) = U({DC}|p^*)
\end{equation}
\begin{equation}
\begin{split}
    p^*u({Go, DR})+(1-p^*)u({Stop, DR}) =\\
    p^*u({Go, DC})+(1-p^*)u({Stop, DC})
\end{split}
\end{equation}

Solving for $p^*$ using the utilities from Table~\ref{tab: utility} yields:
\begin{equation}
\label{eq:ppp}
\begin{split}
    p^* &= \frac{u({Stop, DC})-u({Stop, DR})}{u({Go, DR})-u({Go, DC}) - u({Stop, DR}) + u({Stop, DC})}\\
    &=\frac{-1 - 1}{-r - 0 - 1 + (-1)} = \frac{2}{2+r}
\end{split}
\end{equation}

To make HV prefer \textit{DC} than \textit{DR}, the optimal signal $s^*$ must generate a posterior belief $p_\theta({Go}|s^*) \ge p^*$, where the minimal nudge distance $s^*$ achieves equality. By setting Equation~\eqref{eq:pp} equal to $p^*$ and solving for $s$, we arrive at the following theorem.

\begin{theorem}[Optimal Signal Value]
The minimum embodied signal $s^*$ required to persuade an HV with trust level $\theta \ge \theta^*$ is given by:
\begin{equation}
\label{eq:s^*}
    s^*(\theta)=L^{-1} \left(\frac{(1-\lambda)(p^*-(1-\theta)\lambda)}{\lambda(\theta+(1-\theta)\lambda-p^*)} \right)
\end{equation}
where $L^{-1}$ is the inverse of the likelihood ratio function, $\lambda=p_0({Go})$, and $p^*=2/(2+r)$.
\end{theorem}

\subsubsection{Remark on the Connection Between $\theta^*$ and $p^*$}

The trust threshold $\theta^*$ and the posterior threshold $p^*$ are linked by the following identity:
\begin{equation}
    p^* \equiv (1-\theta^*)p_0({Go})+\theta^*p_{FT}({Go}|s_{max})
\end{equation}
where $p_{FT}({Go}|s_{max}) = 1$ is the posterior belief under a maximally informative signal.

This identity reveals a fundamental relationship between the belief-space objective of persuasion and the Receiver's intrinsic level of trust. The threshold posterior, $p^*$, represents the critical belief at which the Receiver becomes indifferent between her default and alternative actions; it effectively serves as the persuasion boundary in the belief space. The trust threshold, $\theta^*$, quantifies the minimum level of receptiveness required for the Receiver's belief to be capable of reaching this boundary.

The identity demonstrates that for a Receiver possessing the minimal trust level $\theta^*$, persuasion is only attainable under the most favorable signaling conditions—that is, when presented with a maximally informative signal $s_{max}$ that perfectly resolves uncertainty in favor of the \textit{Go} state. For any trust level $\theta < \theta^*$, the belief state $p^*$ is unreachable, regardless of the signal's content. One can verify this identity by substituting the values derived in Equation~\eqref{eq:tt} and Equation~\eqref{eq:ppp}.

\subsection{TA-EBP}
\label{sec:TAEBP}
We now synthesize the preceding theorems into the complete {Trust-Aware Embodied Bayesian Persuasion (TA-EBP)} framework. This framework provides a principled and transparent method for an AV to compute and commit to an optimal signaling strategy in order to enhance the safety by making more HVs to drive if clear.

For an HV characterized by a prior belief $p_0(Go) = \lambda$, a trust level $\theta$, and a known utility parameter $r$, the TA-EBP algorithm proceeds as follows:

\paragraph{Step 1: Verify Persuadability}
First, the AV calculates the minimum persuadable trust level $\theta^*$ using the result from our first theorem (Equation~\eqref{eq:theta_star}). It then verifies the persuadability condition:
\begin{itemize}
    \item If $\theta < \theta^*$, the HV is deemed unpersuadable. The AV should revert to a default safe policy (e.g., always yielding).
    \item If $\theta \ge \theta^*$, the HV is persuadable, and the AV proceeds to the next step.
\end{itemize}

\paragraph{Step 2: Construct the Optimal Signaling Scheme}
If the HV is persuadable, the AV constructs the optimal signaling scheme $\pi$ before knowing the state of the world $\omega$. The signal space is $\mathcal{S}=\{s^*, 0\}$, where $s^*$ is the optimal nudge magnitude derived in Theorem 2 (Equation~\eqref{eq:s^*}). The signaling scheme takes the form similar to the prosecutor and judge example introduced in Section~\ref{sec:BP_example}:
\begin{align*}
    \pi(s^*|\text{Go}) = a, \quad & \pi(0|\text{Go}) = 1-a \\
    \pi(s^*|\text{Stop}) = b, \quad & \pi(0|\text{Stop}) = 1-b
\end{align*}

To maximize influence, the values of $a$ and $b$ are set as follows:
\begin{itemize}
    \item Set $a=1$: It is physically intuitive that an AV with the intent to \textit{Go} will always perform the corresponding nudge signal. Therefore, we set $\pi(s^*|\text{Go}) = 1$.
    \item Derive $b$: The parameter $b = \pi(s^*|\text{Stop})$ is calibrated to be the maximum possible value that still makes the signal $s^*$ persuasive. This is achieved by ensuring that the posterior belief generated by $s^*$ is exactly equal to the HV's indifference posterior, $p^*$, derived in Equation~\eqref{eq:ppp}. We solve the following for $b$:
    \begin{equation}
    \begin{split}
        p^* &= (1-\theta)p_0(\text{Go}) \\
            &+ \theta \cdot \left(\frac{\pi(s^*|\text{Go})p_0(\text{Go})}{\pi(s^*|\text{Go})p_0(\text{Go}) + \pi(s^*|\text{Stop})p_0(\text{Stop})} \right)
    \end{split}
    \end{equation}
    Substituting $\pi(s^*|\text{Go})=1$, $\pi(s^*|\text{Stop})=b$, and $p_0(\text{Go})=\lambda$ yields:
    \begin{equation}
        p^*=(1-\theta)\lambda+\theta \cdot \left(\frac{\lambda}{\lambda+b(1-\lambda)} \right)
    \end{equation}
    Solving for $b$ gives:
    \begin{equation}
    \label{eq:b}
        b=\frac{\lambda(\theta - p^* + (1-\theta)\lambda)}{(1-\lambda)(p^* - (1-\theta)\lambda)}
    \end{equation}
\end{itemize}

\paragraph{Step 3: Commit and Execute}
The final TA-EBP committed signaling scheme is:
\begin{align*}
    \pi(s^*|\text{Go}) = 1, \quad & \pi(0|\text{Go}) = 0 \\
    \pi(s^*|\text{Stop}) = b, \quad & \pi(0|\text{Stop}) = 1-b
\end{align*}
where $s^*$ is the optimal nudge magnitude from Equation~\eqref{eq:s^*} and $b$ is given by Equation~\eqref{eq:b}.

Under the TA-EBP policy, the AV sends the persuasive signal $s^*$ with a total probability of $P(s^*) = \pi(s^*|\text{Go})p_0(\text{Go}) + \pi(s^*|\text{Stop})p_0(\text{Stop}) = \lambda + b(1-\lambda)$. This represents the fraction of interactions in which the HV is successfully persuaded to choose the safe action \textit{DC}. This outcome is an improvement over a fully revealing policy, which would only result in the safe action in a $\lambda$ fraction of cases.

\begin{figure*}[htbp]
    \centering
    \includegraphics[width=\textwidth]{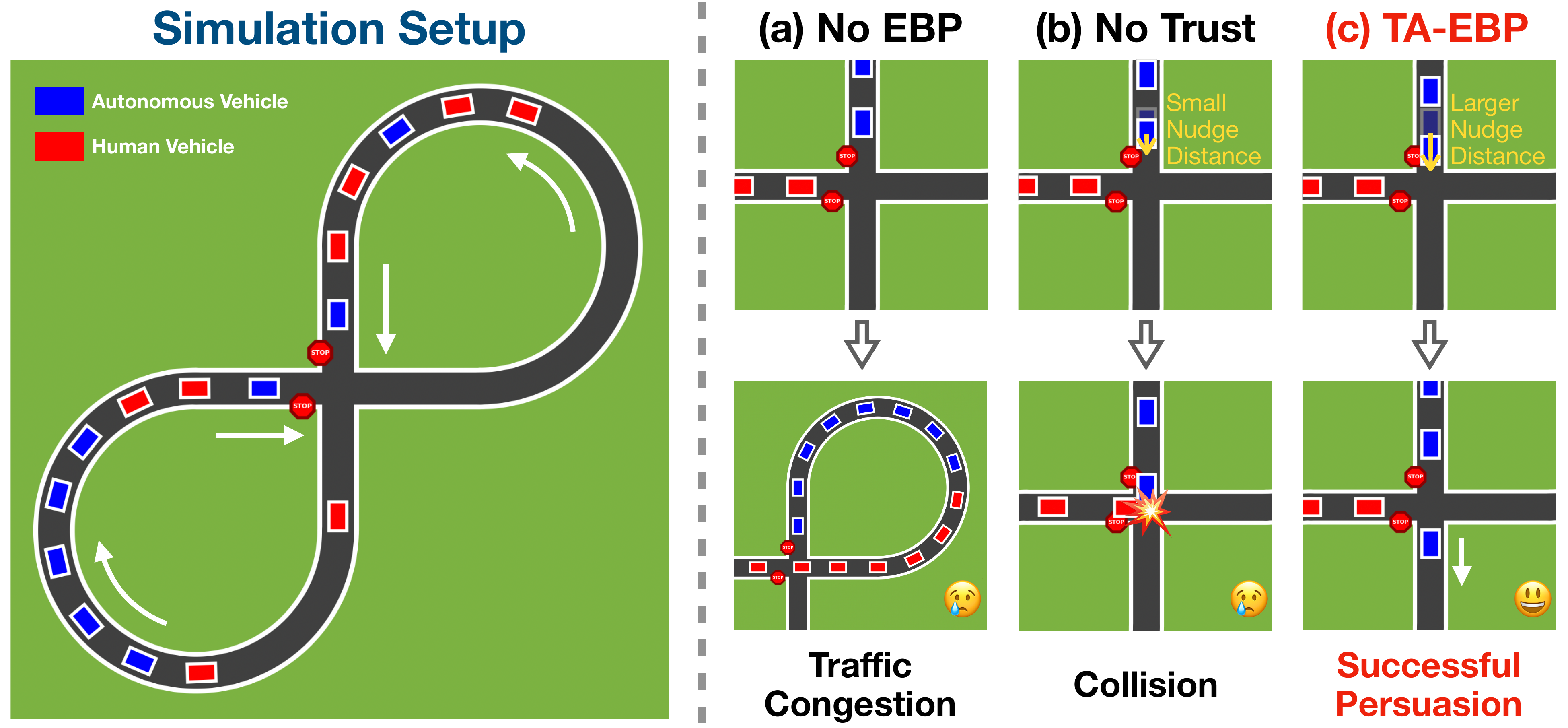} 
    \caption{The mixed-autonomy simulation environment (left) and a comparison of outcomes for the three models (right). (a) The ``No EBP" baseline results in traffic congestion as the AV perpetually yields. (b) The ``No Trust" model's suboptimal nudge fails to persuade the HV, leading to a collision. (c) Our ``TA-EBP" framework uses an optimally calibrated nudge and successfully persuades HV to yield, enhancing safety and efficiency.}
    \label{fig:sim_setup}
\end{figure*}

\section{Simulation}

In this section, we evaluate the practical performance of our TA-EBP framework in a custom mixed-autonomy traffic simulation. We detail the experimental design, define the baseline models for comparison, and analyze the results.

\subsection{Experimental Setup}

We designed a simulation environment featuring a single-lane, figure-8 track with a 2-way stop intersection at its center, as illustrated in the left panel of Figure~\ref{fig:sim_setup}. The environment is populated by eight autonomous vehicles (AVs, in blue) and eight human-driven vehicles (HVs, in red). When an AV and an HV meet at the intersection, the AV employs one of the strategies defined in Section~\ref{sec:baselines}. If two vehicles of the same type meet, they cross based on arrival time. After crossing, vehicles rejoin the opposite traffic queue, creating a continuous flow.

The simulation is parameterized as follows: vehicles have a cruising velocity of 4.0 units per second and maintain a minimum following distance of 2.0 units. At the intersection, vehicles stop 1.7 units from the center. Collisions occur only at the intersection if both vehicles proceed simultaneously. The key simulation model parameters are set as: (1) Prior $p_0(\text{Go})=\lambda = 0.2$. (2) Collision penalty $u(Go, DC)=-r=-3$. (3) Trust level $\theta=0.6$.
We model the human's inherent likelihood belief for a nudge of magnitude $s \in [0, 2]$ as linear functions: $\pi(s|\text{Go}) = s/2$ and $\pi(s|\text{Stop}) = 1 - s/2$. This yields a likelihood ratio $L(s) = s / (2-s)$.

\subsection{Models and Metrics}
\label{sec:baselines}
We compare the performance of our proposed TA-EBP framework against two critical baselines on the same simulation model, depicted in the right panels of Figure~\ref{fig:sim_setup}:
\begin{enumerate}
    \item {No EBP:} A non-communicative model where the AV never sends a signal (i.e., never nudges). Its decisions are based solely on predicting the HV's actions from the prior.
    \item {No Trust (Trust-Agnostic EBP):} An EBP model that does not account for trust. It operates under the classic Bayesian persuasion assumption that the HV is a perfectly rational updater, which is equivalent to assuming the HV's trust is $\hat{\theta}=1$.\footnote{For a fair comparison, the HV agent in all simulations, including this baseline, updates its belief using its true trust level $\theta=0.6$.}
\end{enumerate}

Performance is evaluated using two primary metrics:
\begin{itemize}
    \item {Collision Rate:} The fraction of AV-HV intersection crossings that result in a collision. Note that we assume collisions can only at the intersection, and other situations such as rear-ender are not considered.
    \item {\textit{DC} Rate:} The fraction of AV-HV crossings where the HV is successfully persuaded to ``drive if clear", i.e., $(\text{Number of}\textit{ DC}) / (\text{Number of all crossings})$. A higher \textit{DC} rate, indicating safer driving behavior, is preferred.
\end{itemize}

\subsection{Results and Discussion}

Each model was simulated for one hour, with the outcomes recorded in Table~\ref{tab:sim_results}.

\paragraph{No EBP}
The ``No EBP" baseline, lacking any signaling capability, results in AVs perpetually yielding to HVs. Based on the low prior for \textit{Go} ($\lambda=0.2$), the HV always assumes the AV will yield to her and thus drives recklessly (\textit{DR}). While this passive strategy avoids collisions, it sometimes leads to severe traffic congestion and zero throughput for the AVs whenever they encounter an HV as shown in Figure~\ref{fig:sim_setup}(a), which may significantly undermine the system efficiency. 

\paragraph{No Trust (Trust-Agnostic EBP)}
The ``No Trust" baseline demonstrates the danger of mismatched trust assumptions. The AV operates on a flawed model of the HV, assuming full trust ($\hat{\theta}=1$). This causes it to compute and execute a nudge ($\hat s=1.45$ units) that is insufficient to persuade the truly less-trusting HV ($\theta=0.6$), for whom the optimal signal is $s^*=1.64$ units. Upon observing the suboptimal signal, the HV's posterior belief remains below the indifference threshold ($p_{\theta}(Go|\hat s) < p^*$), which means $U(DR|p_{\theta}(Go|\hat s)) > U(DC|p_{\theta}(Go|\hat s))$ so the HV drives recklessly (\textit{DR}). The AV, however, assumes its persuasion was successful and also proceeds, leading to a collision. This systematic failure would result in a high collision rate of 20\% in theory, as illustrated in Figure~\ref{fig:sim_setup}(b). In our simulation, the 0.203 collision rate proves this theoretical calculation.

\paragraph{TA-EBP}
In contrast, our TA-EBP framework correctly accounts for the HV's actual trust level. It generates an optimally calibrated nudge ($s^*=1.64$ units) that successfully raises the HV's posterior belief to the persuasion threshold. As a result, in the simulation TA-EBP achieves a zero collision rate while persuading 38.0\% of HVs to choose \textit{DC}. In theory as calculated in Section~\ref{sec:TAEBP} the \textit{DC} rate should be 37.6\%, which agrees with our simulation result. As another comparison, our TA-EBP framework achieves a 17.6 percentage point increase in safe human driving behavior over a fully informative (non-persuasive) policy, which would only achieve a \textit{DC} rate equal to the prior of 20\% in theory. Our framework thus enhances safety and facilitates efficient, successful crossings, as seen in Figure~\ref{fig:sim_setup}(c).

\begin{table}[htbp]
\centering
\caption{\textnormal{\small{Performance comparison of the three frameworks in simulation.}}}
\begin{tabular}{lccc}
\toprule
\textbf{Metric} & \textbf{No EBP} & \textbf{No Trust} & \textbf{TA-EBP (Ours)} \\ \midrule
Collision Rate  & 0       & 0.203     & \textbf{0}        \\
\textit{DC} Rate         & 0       & 0         & \textbf{0.380}    \\ \bottomrule
\end{tabular}
\label{tab:sim_results}
\vspace{-8pt}
\end{table}

\section{Conclusion and Discussion}

\subsection{Conclusion}
In this work, we addressed the challenge of enhancing safety and efficiency in mixed-autonomy traffic intersections. Observing that traditional game-theoretic approaches can suffer from a decay of influence and can be perceived as manipulative, we proposed a novel approach that seeks to shape human driver beliefs in a transparent and non-strategic manner. Our primary contribution is the development of the Trust-Aware Embodied Bayesian Persuasion (TA-EBP) framework. This was achieved by making two key advancements to the classical Bayesian persuasion model. First, to account for human behavior variance, we introduced a trust parameter and derived a theorem for the minimum persuadable trust level, providing a formal condition for when influence is possible. An analysis of the affects of individual prior and collision penalty is also provided. Second, in contrast to prior work that uses abstract or discrete signals, we grounded the signaling mechanism in the physical world, developing a theorem that determines the optimal magnitude of an embodied signal. Lastly we synthesized these contributions into a complete, step-by-step framework for direct application. To validate our approach, we conducted simulations comparing TA-EBP against two baselines, and the results demonstrate that our framework significantly enhances both traffic safety and efficiency, successfully persuading human drivers to cross the intersection more cautiously.

\subsection{Limitations and Future Work}
This work represents a foundational step in applying Bayesian persuasion to embodied AI, and several exciting avenues for future research remain. First, while designed for a specific traffic scenario, the TA-EBP framework is broadly applicable. Future work should investigate its adaptation to other human-centered robotics contexts. For instance, a collaborative robot on an assembly line could use a subtle pre-movement of its arm as an embodied signal to persuade a human co-worker to prepare for a handover, thereby improving workflow fluency. Furthermore, our current model assumes a static trust level. A promising direction is to integrate a dynamic trust model, where a user's trust is updated after each interaction based on the outcome, as explored by \cite{Guo2021ModelingAP}. This would allow the agent's signaling scheme to dynamically adapt to the evolving human-robot relationship. Finally, while our simulation results are promising, the framework's effectiveness must be validated with human participants. Future work should involve comprehensive user studies in high-fidelity simulators or real-world settings to assess TA-EBP with a diverse human population.

\section*{ACKNOWLEDGMENT}
This work was supported by the National Science Foundation under Grant CCF 2236484.




\bibliographystyle{IEEEtran}
\bibliography{references.bib}

\end{document}